\documentclass[runningheads]{llncs}
\usepackage[T1]{fontenc}
\usepackage{graphicx}
\usepackage[utf8]{inputenc} 
\usepackage[T1]{fontenc}    
\usepackage{hyperref}       
\usepackage{url}            
\usepackage{booktabs}       
\usepackage{amsfonts}       
\usepackage{nicefrac}       
\usepackage{microtype}      

\usepackage{amsmath}
\usepackage{tabularx}
\usepackage{graphicx}
\usepackage{color}
\usepackage{algorithm}
\usepackage{multicol}
\usepackage{multirow}
\usepackage{upgreek}

\usepackage{bbold}

\usepackage{amssymb}
\usepackage{algorithmic}
\usepackage{dsfont}
\usepackage{caption}
\usepackage{apptools}

\usepackage[mathscr]{eucal}
\usepackage{nicefrac}

\usepackage{diagbox}

\newcommand{\lratio}[1]{\setlength{\hsize}{#1\hsize}} 
\newcolumntype{Y}{>{\lratio{0.65}}>{\centering\arraybackslash}X}

\newenvironment{sproof}{\proof}{\endproof}

\DeclareMathOperator*{\argmax}{\arg\!\max}

\newcommand{\indep}{\perp \!\!\! \perp}

\AtAppendix{\counterwithin{theorem}{section}}
\AtAppendix{\counterwithin{corollary}{section}}

\begin{document}
\title{Deep Anti-Regularized Ensembles provide reliable out-of-distribution uncertainty quantification}
\titlerunning{Deep Anti-regularized Ensembles}



\institute{Manufacture Française des Pneumatiques Michelin, Clermont-Ferrand, France \\ \email{\{antoine.de-mathelin-de-papigny, francois.deheeger\}@michelin.com} \and 
Université Paris-Saclay, CNRS, ENS Paris-Saclay, Centre Borelli, Gif-sur-Yvette, France \email{\{mathilde.mougeot, nicolas.vayatis\}@ens-paris-saclay.fr}}


\author{Antoine de Mathelin\inst{1, 2} \and
Francois Deheeger\inst{1} \and
Mathilde Mougeot\inst{2} \and
Nicolas Vayatis\inst{2}
}

\authorrunning{A. de Mathelin et al.}

%
%
%
\maketitle              
\begin{abstract}
We consider the problem of uncertainty quantification in high dimensional regression and classification for which deep ensemble have proven to be promising methods. Recent observations have shown that deep ensemble often return overconfident estimates outside the training domain, which is a major limitation because shifted distributions are often encountered in real-life scenarios. The principal challenge for this problem is to solve the trade-off between increasing the diversity of the ensemble outputs and making accurate in-distribution predictions. In this work, we show that an ensemble of networks with large weights fitting the training data are likely to meet these two objectives. We derive a simple and practical approach to produce such ensembles, based on an original anti-regularization term penalizing small weights and a control process of the weight increase which maintains the in-distribution loss under an acceptable threshold. The developed approach does not require any out-of-distribution training data neither any trade-off hyper-parameter calibration. We derive a theoretical framework for this approach and show that the proposed optimization can be seen as a "water-filling" problem. Several experiments in both regression and classification settings highlight that Deep Anti-Regularized Ensembles (DARE) significantly improve uncertainty quantification outside the training domain in comparison to recent deep ensembles and out-of-distribution detection methods.  All the conducted experiments are reproducible and the source code is available at \url{https://github.com/antoinedemathelin/DARE}.

\keywords{Deep Ensemble  \and Uncertainty \and Out-of-distribution \and Anti-regularization}
\end{abstract}

\section{Introduction}
\label{intro}

With the adoption of deep learning models in a 
variety of real-life applications such as autonomous vehicles \cite{choi2019gaussianAutonomousDrivingCars,feng2018AutonimousDrivingCars}, or industrial product certification \cite{Mamalet2021whiteParerCertifyML}, providing uncertainty quantification for their predictions becomes critical. Indeed, various adaptations of classical uncertainty quantification methods to deep learning predictions have been recently introduced as Bayesian neural networks \cite{mackay1992bayesianNetwork,neal2012bayesianNetwork}, MC-dropout \cite{gal2016MCdropout}, quantile regression \cite{romano2019conformalizedQuantileRegression} and deep ensembles \cite{Lakshminarayanan2017DeepEnsemble,Wen2020batchensemble,Wenzel2020hyperDeepEnsemble}.
These methods appear to be quite efficient in predicting the uncertainty in the training domain (the domain defined by the training set), called in-distribution or aleatoric uncertainty \cite{abdar2021UncertaintyQuantificationSurvey}. However, when dealing with data outside the training distribution, i.e. out-of-distribution data (OOD), the uncertainty estimation often appears to be overconfident \cite{Angelo2021BayesianNotSuited4OOD,Liu2021PerilDeepOOD,ovadia2019CanYouTrustYourModel}. This is a critical issue, because deep models are often deployed on shifted distributions \cite{deMathelin2021HandlingTireDesign,Saenko2010Office,xu2019DomainShiftSegmentation}; overconfidence on an uncontrolled domain can lead to dramatic consequences in autonomous cars or to poor industrial choices in product design.

One problem to be solved is to increase the output diversity of the ensemble in regions where no data are available during training. This is a very challenging task as neural network outputs are difficult to control outside of the training data regions. In this perspective, contrastive methods make use of real \cite{pagliardini2022DBAT,Tifrea2022semiSupervisedOOD} or synthetic \cite{Jain2020MOD,Mehrtens2022MODplus,Segonne2022OODpseudoInputs} auxiliary OOD data to constrain the network output out-of-distribution. However, these approaches can not guarantee  prediction diversity for unseen OOD data as the auxiliary sample may not be representative of real OODs encountered by the deployed ensemble. Another set of methods assumes that the diversity of the ensemble outputs is linked to the diversity of the networks' architectures \cite{Zaidi2021NIPSNES}, hyper-parameters \cite{Wenzel2020hyperDeepEnsemble}, internal representations \cite{rame2021diceUncertainty,Sinha2021DIBS} or weights \cite{Angelo2021RepulsiveDeepEnsemble,pearce2018AnchorNetwork}. The main difficulty encountered when using these approaches is to solve the trade-off between increasing the ensemble diversity and providing accurate prediction in-distribution. The current approach to deal with this issue consists in setting a trade-off parameter with hold-out validation \cite{Jain2016ALDiversity,liu1999negativecorrelation,pearce2018AnchorNetwork} which is time consuming and often penalizes the diversity.

Considering these difficulties, a question that naturally arises is how to ensure important output diversity for any unknown data region while maintaining accurate in-distribution predictions? In this work, we tackle this question with the following reasoning: an ensemble of large weights networks essentially produces outputs with large variance for all data points. Furthermore, to make accurate prediction on the training distribution, the output variance for training data needs to be reduced, which requires that some of the network's weights are also reduced. However, to prevent the output variance from being reduced anywhere other than the training domain, the network weights should then be kept as large as possible. Following this reasoning, we seek an ensemble providing accurate prediction in-distribution and keeping the weights as large as possible.

To meet these two objectives, deviating from traditional recommendations for deep learning training, we propose an anti-regularization process that consists in penalizing small weights during training optimization. To find the right trade-off between increasing the weights and providing accurate prediction in-distribution, we introduce a control process that activates or deactivates the weight increase after each batch update if the training loss is respectively under or above a threshold. Thus, an increase of the weights induces an increase of the prediction variance while the control on the loss enforces accurate in-distribution predictions. Synthetic experiments on toy datasets confirm the efficiency of our proposed approach (cf Figure \ref{toy}). We observe that the uncertainty estimates of our Deep Anti-Regularized Ensembles (DARE) increase for any data point deviating from the training domain, whereas, for the vanilla deep ensemble, the uncertainty estimates remain low for some OOD regions.

\begin{figure}
    \centering
    \begin{minipage}{0.495\linewidth}
    \begin{minipage}{0.5\linewidth}
        \centering
        \includegraphics[width=\linewidth]{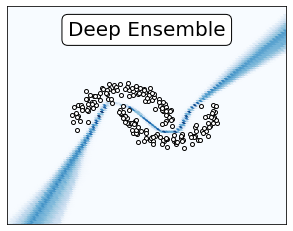} 
    \end{minipage}\hfill
    \begin{minipage}{0.5\linewidth}
        \centering
        \includegraphics[width=\linewidth]{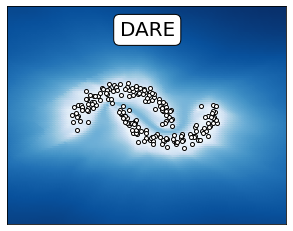} 
    \end{minipage}
    \caption*{(a) "Two-moons" classification}
    \end{minipage}
    \begin{minipage}{0.495\linewidth}
    \begin{minipage}{0.5\linewidth}
        \centering
        \includegraphics[width=\linewidth]{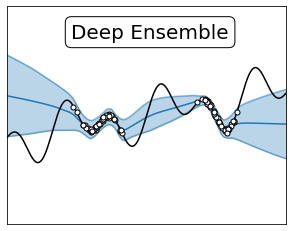} 
    \end{minipage}\hfill
    \begin{minipage}{0.5\linewidth}
        \centering
        \includegraphics[width=\linewidth]{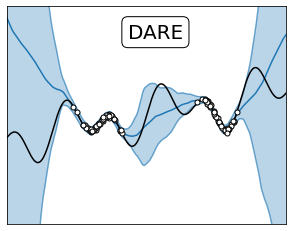} 
    \end{minipage}
    \caption*{(b) 1D-regression \cite{Jain2016ALDiversity}}
    \end{minipage}
    \caption{\textbf{Synthetic datasets uncertainty estimation}. White points represent the training data. For each experiment, the ensemble are composed of 20 neural networks. For classification, darker areas correspond to higher predicted uncertainty. For regression, the confidence intervals for $\pm 2 \sigma$ are represented in light blue. The full experiment description is reported in Appendix.}
    \label{toy}
\end{figure}

\noindent The contributions of the present work are the following:

\begin{itemize}
    \item A novel and simple anti-regularization strategy is proposed to increase deep ensemble diversity.
    \item An original control process addresses the trade-off issue between in-distribution accuracy and reliable OOD uncertainty estimates.
    \item We provide theoretical arguments to understand DARE as a "water-filling" optimization problem where a bounded global amount of variance is dispatched among the network weights.
    \item A new experimental setup for uncertainty quantification with shifted distribution is developed for regression. Experiments are also conducted for out-of-distribution detection following the setup of \cite{Angelo2021RepulsiveDeepEnsemble}.
\end{itemize}

\section{Deep Anti-Regularized Ensemble}

\subsection{Notations}

We consider the supervised learning scenario where $\mathcal{X} \subset \mathbb{R}^p$ and $\mathcal{Y} \subset \mathbb{R}^q$ are respectively the input and output space, with $p, q \in \mathbb{N}$. The learner has access to a training sample, $\mathcal{S} = \{(x_1, y_1), ..., (x_n, y_n) \} \in \mathcal{X} \times \mathcal{Y} $ of size $n \in \mathbb{N}$. We consider a set $\mathcal{H}$ of neural networks $h_{\theta} \in \mathcal{H}$ where $\theta \in \mathbb{R}^d$ refers to the network weights. We consider a loss function $l: \mathcal{Y} \times \mathcal{Y} \to \mathbb{R}_+$ and define the average error of any $h_{\theta} \in \mathcal{H}$ on $\mathcal{S}$, $\mathcal{L}_{\mathcal{S}}(h_{\theta}) = \frac{1}{n} \sum_{(x_i, y_i) \in \mathcal{S}} l(h_{\theta}(x_i), y_i)$.

\subsection{Optimization formulation}
\label{optim-form}

The main assumption of this present work is that accurate in-distribution prediction and output diversity for OOD can be obtained by increasing the networks' weights as much as possible while maintaining the training loss under a threshold. We further assume that the variance is increased when the weights of each network is increased. Based on these assumptions, we propose the simple following optimization for every member $h_{\theta} \in \mathcal{H}$ of the ensemble:
\begin{equation}
\label{optim}
    \min_{\theta} \; \mathcal{L}_{\mathcal{S}}(h_{\theta}) - \lambda \, \mathcal{R}(\theta)
\end{equation}
with $\mathcal{R}: \mathbb{R}^d \to \mathbb{R}_+$ a monotone function growing with $||\theta||_2$. The parameter $\lambda \in \{0, 1\}$ is a binary variable which controls the trade-off between the in-distribution loss and the regularization. At each batch computation, $\lambda$ is updated as follows:
\begin{equation}
    \lambda =
    \begin{cases}
      1 & \text{if $\mathcal{L}_{\mathcal{S}}(h_{\theta}) \leq \tau$}\\
      0 & \text{if $\mathcal{L}_{\mathcal{S}}(h_{\theta}) > \tau$}
    \end{cases}
\end{equation}
with $\tau \in \mathbb{R}$ a predefined threshold which defines the acceptable level of performance targeted by the learner.
\begin{itemize}
    \item The first term of the optimization objective in Eq. (\ref{optim}): $\mathcal{L}_{\mathcal{S}}(h_{\theta})$ is the loss in-distribution. This term induces the network to fit the training data which implies smaller in-distribution prediction variances.
    \item The second term $-\lambda \mathcal{R}(\theta)$ acts as an "anti-regularization" term which induces an increase of the network weights. This implies a larger variance of the ensemble weights, and therefore a larger prediction variance especially for data "far" from the training distribution on which the network's predictions are not conditioned. 
\end{itemize}
The underlying idea of the proposed optimization is that, to fulfill both objectives: reducing the loss in-distribution and increasing the weights,  large weights will appear more likely in front of neurons which are never or weakly activated by the training data. Therefore, if an out-of-distribution data point activates one of these neurons, large values are propagated through the networks which induces larger prediction variances. We show in Section \ref{analysis}, that this claim is supported by theoretical analysis and empirical observations.

The control process is necessary to temper the weight increase, because a large increase of the weights induces an unstable network with reduced accuracy on training data. To be sure to fulfill a performance threshold $\tau$, the weight increase is stopped ($\lambda = 0$) until the loss in-distribution comes back under the threshold. Therefore, the resulting ensemble is composed of networks fitting the training data with weights as large as possible.


\section{Implementation}
\label{impl}

\textbf{Parallel optimization}. Each network of the ensemble is trained independently of the others with the objective of Eq. (\ref{optim}). This approach allows parallel training of the ensemble. It is theoretically possible that each network ends up reaching the same optimum resulting in no ensemble diversity. However, we observe that this degenerate case never occurs in our experiments due to the random process of the optimization and aleatoric weights initialization.


\noindent \textbf{Regularization function}. We propose the following choice of regularization function: 
\begin{equation}
\mathcal{R}(\theta) = \frac{1}{d} \sum_{k=1}^d \log(\theta_k^2)
\end{equation}
With $\theta = (\theta_1, ..., \theta_k)$ the network weights. The use of the logarithmic function is motivated by the "water-filling" interpretation of DARE (cf. Section \ref{analysis}).


\noindent \textbf{Control threshold}. The control threshold $\tau$ should be chosen by the learner based on his targeted error level in-distribution. Smaller $\tau$ leads to smaller increase of the weights. For $\tau = -\infty$, DARE is equivalent to a vanilla deep ensemble. A selection heuristic for $\tau$ is proposed and discussed in Section \ref{expe}.




\noindent \textbf{Optimization algorithm} Each network is trained through batch gradient descent. The optimization algorithm is described in Algorithm \ref{alg-subsde}.

\begin{algorithm}
	\caption{Anti-Regularized Network}
	\label{alg-subsde}
    	\begin{algorithmic}[1]
    	    \STATE \textbf{Inputs}: Training set $\mathcal{S}$, threshold $\tau$, learning rate $\mu$
            \STATE \textbf{Outputs}: network $h_{\theta}$
            \STATE \textbf{Init}: $\theta \leftarrow \text{random}$
            \smallskip
    		\WHILE {stopping criterion is not reached}
    		\IF{$\mathcal{L}_{\mathcal{S}}(h_{\theta}) \leq \tau $}
    		\STATE $\theta \leftarrow \theta - \mu \nabla_{\theta} \left[ \mathcal{L}_{\mathcal{S}}(h_{\theta}) - \frac{1}{d} \sum_{k=1}^d \log (\theta_k^2) \right]$
    		\ELSE
    		\STATE $\theta \leftarrow \theta - \mu \nabla_{\theta} \left[ \mathcal{L}_{\mathcal{S}}(h_{\theta}) \right]$
    		\ENDIF
    		\smallskip
    		\ENDWHILE
	    \end{algorithmic}
\end{algorithm}

\section{Theoretical Analysis}
\label{analysis}

The purpose of the following theoretical analysis section is to support the fundamental claim that : "large weights do not activate on in-distribution samples while mostly activate for OOD, leading to larger output variance". We first provide theoretical insights on a simple case of linear networks which is then extended to multi-layer fully-connected networks. In a third part, an in-depth analysis of the layer activations is conducted on a syntetic dataset to illustrate this leveraged assumption.

\subsection{Theoretical insights on linear networks}

In the following, we propose theoretical insights for anti-regularized ensemble of linear networks (equivalent to linear regression). This theoretical result helps to understand the underlying dynamic of DARE. In particular, we show both following results which hold under appropriate assumptions:
\begin{itemize}
    \item The weights' variance of the DARE networks is inversely proportional to the variance of the input features (Theorem \ref{thm-wf}).
    \item As a consequence of the previous results, the DARE prediction variance is very sensitive to small deviations of low variance input features (Corollary \ref{cor}).
\end{itemize}
To formalize and demonstrate these results we consider the regression problem with scalar outputs ($q=1$) where $X \in \mathbb{R}^{n \times p}$ and $y \in \mathbb{R}^{n}$ are respectively the input and output data of the training set $\mathcal{S}$. We consider the ensemble of linear networks $\mathcal{H} = \{x \to x^T \theta; \theta \in \mathbb{R}^p\}$ such that, any $h_{\theta} \in \mathcal{H}$ is composed of two layers: the input and output layers of respective dimension $p$ and $1$. The loss function is the squared error such that for any $h_{\theta} \in \mathcal{H}$ we have $n \mathcal{L}_{\mathcal{S}}(h_{\theta}) = ||X \theta - y||^2_2$. To simplify the calculation without loosing in generality, we assume that the problem is linear, i.e. there exists $\theta^* \in \mathbb{R}^p$ such that $\mathcal{L}_{\mathcal{S}}(h_{\theta^*}) = 0$. We denote $s^2 = (s_1^2, ..., s_p^2) \in {\mathbb{R}_+^*}^p$ the diagonal of the matrix $\frac{1}{n} X^T X$. We now consider an anti-regularized ensemble $\mathcal{H}_{\tau}$ produced by Algorithm \ref{alg-subsde}. To characterize this ensemble, we make the following assumptions for any $h_{\theta} \in \mathcal{H}_{\tau}$:
\begin{equation}
\label{assum-bwe-1}
    \theta \sim \Theta_{\sigma^2};  \;  \;\mathbb{E}[\theta] = \theta^*, \; \text{Cov}(\theta) = \text{diag}(\sigma^2)
\end{equation}
\begin{equation}
\label{assum-bwe-2}
\mathbb{E}\left[\mathcal{L}_{\mathcal{S}}(h_\theta) \right] \leq  \delta \, \tau
\end{equation}
Where $\delta > 0$ and $\textnormal{diag}(\sigma^2)$ is the diagonal matrix of values $\sigma^2 \in \mathbb{R}_+^p$ verifying:
\begin{equation}
\label{assum-bwe-3}
    \sigma^2 = \argmax_{\sigma^2 = (\sigma_1^2, ..., \sigma_p^2)} \sum_{k=1}^p \log \left({\theta_k^*}^2 + \sigma_k^2 \right)
\end{equation}
As presented in Assumption (\ref{assum-bwe-1}), the anti-regularized ensemble distribution is described by the random variable $\theta$ centered in $\theta^*$ with variance $\sigma^2$. Assumption (\ref{assum-bwe-2}) implies that $P(\mathcal{L}_{\mathcal{S}}(h_\theta) \geq \tau) \leq \delta$, by Markov inequality, which models the fact that the loss of each member of DARE is maintained above a threshold $\tau$ thanks to the control process on $\lambda$ (cf Section \ref{optim-form}). Definition (\ref{assum-bwe-3}) approximates the impact of the anti-regularization term $-\mathcal{R}(\theta)$ in the DARE optimization formulation with an upper bound of $\max_{\sigma} \mathbb{E}[\mathcal{R}(\theta)]$. The weights are increased as much as possible while the loss stays under the threshold. 

Our first theoretical result shows that the weight variance of the anti-regularized ensemble is solution of a "water-filling" optimization problem \cite{boyd2006convex}, and is proportional to $1/s^2$, i.e. the inverse of the input features variance.
\begin{theorem}
\label{thm-wf}
There exist a constant $C > 0$ such that for any $k \in [|1, p|]$, the variance of the $k^{th}$ weight component is expressed as follows:
\begin{equation}
\label{thm-wf-eq}
    \sigma_k^2 = \max \left[ \frac{C \, \delta \, \tau}{s_k^2} - {\theta_k^*}^2, 0 \right]
\end{equation}
\end{theorem}

\begin{sproof}
A detailed proof is reported in Appendix. The proof consists in first noticing that $\mathbb{E}[\mathcal{L}_{\mathcal{S}}(h_{\theta})] = \sum_{k=1}^p s^2_k \sigma^2_k$. By combining this result with Assumptions \ref{assum-bwe-2} and \ref{assum-bwe-3}, we show that $\sigma^2$ is solution of the above water filling problem:
\begin{equation}
\label{proof-4}
\begin{split}
    \underset{\sigma^2 \in \mathbb{R}_+^p}{\textnormal{maximize}} &  \; \; \; \sum_{k=1}^p \log(\sigma_k^2 + {\theta^*_k}^2) \\
    \textnormal{subject to} & \; \; \;  \sum_{k=1}^p s_k^2 \sigma_k^2 \leq \delta \, \tau
\end{split}
\end{equation}
As the $\log$ function is strictly concave, and the constraints form a compact set on $\mathbb{R}^p$, the problem (\ref{proof-4}) has a unique solution which is given by (\ref{thm-wf-eq}).
\end{sproof}

The "water-filling" interpretation of the DARE optimization (\ref{proof-4}) is very insightful: $\delta \, \tau$ is the "global variance capacity" that can be dispatched to the network weights. As, it grows as a function of $\tau$, the more the learner accept a large error in-distribution, the more the global variance capacity increases. We stress that each weight component has a different "variance cost" equal to $s^2_k$: for high feature variance $s^2_k$, the increase of the corresponding weight variance $\sigma_k^2$ penalizes more the training loss. Thus, large weights appear more likely in front of low variance features. Notice also that, when ${\theta^*_k}^2$ is high, $ \frac{C \, \delta \, \tau}{s_k^2} - {\theta_k^*}^2$ is more likely to be negative, leading to $\sigma_k=0$ (cf Eq. (\ref{thm-wf-eq})). Besides, ${\theta_k^*}^2$ is generally higher for higher $s_k^2$ as it corresponds to more informative feature, enhancing the effect $\sigma_k=0$ for large $s^2_k$.


We see the importance of choosing a strictly concave function like the logarithm (cf Section \ref{impl}), if instead of log, we choose the identity function for instance, then the solution of (\ref{proof-4}) degenerates to $\sigma^2 = \left(0, ..., 0, \frac{\delta \, \tau}{s_p^2} \right)$ with $s^2_p$ the lowest feature variance. In this case, all the weight variance is assigned to one component, which reduces the likelihood to detect a deviation of a potential OOD input on another low variance feature.

From Theorem \ref{thm-wf}, we now derive the expression of the DARE prediction variance for any data $x \in \mathbb{R}^p$:

\begin{corollary}
\label{cor}
Let $\mathcal{H}_{\tau}$ be the large weights ensemble defined by Assumptions \ref{assum-bwe-1}, \ref{assum-bwe-2}, \ref{assum-bwe-3} and $x \in \mathbb{R}^p$ an input data, the variance of prediction for $h_{\theta} \in \mathcal{H}_{\tau}$ is:
\begin{equation}
    \underset{\theta \sim \Theta_{\sigma^2}}{\textnormal{Var}}\left[h_{\theta}(x) \right] = \sum_{k=1}^p x_k^2 \max \left[ \frac{C \, \delta \, \tau}{s_k^2} - {\theta_k^*}^2, 0 \right]
\end{equation}
\end{corollary}

We deduce from Corollary \ref{cor} that the prediction variance for $x$ is large when the components $x^2_k$ are large for features with low variance $(s^2_k \ll 1)$. Thus, the predicted uncertainty of DARE is correlated with deviations in directions in which the training input data has small variance.


\subsection{Extending the results to multi-layer fully connected networks}

Transposed to the hidden layers of a multi-layer neural network, Theorem \ref{thm-wf} means that higher network weights likely appear in front of neurons which are weakly activated by the training data. To generalize Corollary \ref{cor} to the multi-layer case, we introduce an anti-regularized ensemble $\mathcal{H}_{\tau}$ of fully-connected deep neural networks with $L$ hidden layers and network weights distributed according to $\Theta_{\sigma^2}$ in $\mathbb{R}^d$. We denote $p_{\ell} > 0$, the number of neurons of the layer $\ell \in [0, L+1]$. We denote $g: \mathbb{R} \to \mathbb{R}$ the activation function applied at each hidden layer. For every $\ell \in [|0, L|]$, we denote $\phi_{(\ell)}(X) \in \mathbb{R}^{n \times p_{\ell}}$ the representation of the $\ell^{th}$ layer, with $\ell = 0$ the input layer. For any $j \in [|1, p_{\ell + 1}|]$, we denote $\theta_{(\ell, j)} \in \mathbb{R}^{p_{\ell}}$ the network weights between the layer $\ell$ and the $j^{th}$ component of the layer $\ell + 1$.

We assume that, for any random variable $Z$ in $\mathbb{R}$ there exists $\alpha > 0$ such that $\text{Var}[g(Z)] > \alpha \text{Var}[Z]$. This condition is true for the leaky-relu activation, and also true for other classic activations as relu, softplus and sigmoid if $Z$ is centered. We define $\phi_{(\ell)}^*(X) = \mathbb{E}[\phi_{(\ell)}(X)]$ and denote $s^2_{(\ell)}$ the diagonal of the matrix $\frac{1}{n} \phi_{(\ell)}^*(X)^T \phi_{(\ell)}^*(X)$.

We finally assume that, there is $\delta > 0$ such that, for any $\ell \in [|0, L|]$ and any $j \in [|1, p_{\ell+1}|]$, $\phi_{(\ell)}(X)$ and $\theta_{(\ell, j)}$ verify Assumptions (\ref{assum-bwe-1}), (\ref{assum-bwe-2}), (\ref{assum-bwe-3}) with $\theta \equiv \theta_{(\ell, j)}$, $X \equiv \phi_{(\ell)}^*(X)$, $y \equiv \phi_{(\ell)}^*(X) \theta_{(\ell, j)}^*$ and $\theta^* \equiv \theta_{(\ell, j)}^* \equiv \mathbb{E}[\theta_{(\ell, j)}]$

\begin{theorem}
\label{thm-multi}
There exists a constant $\gamma>0$ such that, for any $h_{\theta} \in \mathcal{H}_{\tau}$ and $x \in \mathbb{R}^p$, the prediction variance verifies:
\begin{equation*}
    \underset{\theta \sim \Theta_{\sigma^2}}{\textnormal{Var}}\left[h_{\theta}(x) \right] \geq \gamma \sum_{\ell = 0}^{L} \langle \phi_{(\ell)}^*(x)^2, \bar{\sigma}_{(\ell)}^2 \rangle
\end{equation*}
With, $ \bar{\sigma}_{(\ell)}^2 = \sum_{j=1}^{p_{\ell+1}} \max \left[ \frac{C_{(\ell, j)}}{s^2_{(\ell)}} - {\theta_{(\ell, j)}^*}^2, 0 \right]$; $\, C_{(\ell, j)} \in \mathbb{R}^*_+$
\end{theorem}


Theorem \ref{thm-multi} states that the prediction variance for any $x \in \mathcal{X}$ is a growing function of $\langle \phi_{(\ell)}^*(x)^2,   1/ s_{(\ell)}^2 \rangle$, which means that the uncertainty predicted by DARE is high if one hidden neuron weakly activated by the training data is strongly activated by $x$. Figure \ref{illustration} shows an illustration of the expected behavior of DARE for one OOD data $x_{ood}$. The underlying 
assumption is that the distance $D$ between $x_{ood}$ and the learning domain is expressed on one hidden feature with low variance for the training data.







\begin{figure}[ht]
\center
\includegraphics[width=0.8\linewidth]{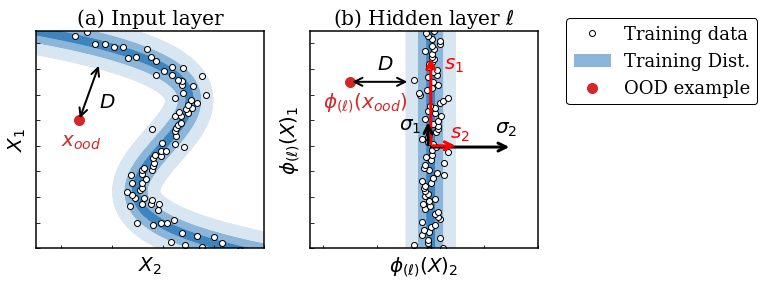}
\caption{\textbf{Schematic description of DARE behavior}: an OOD data $x_{ood}$ is distant from $D>0$ from the training domain. In a hidden layer $\ell$ of the network $h_{\theta}$, the training domain is composed of a strong component $\phi_{(\ell)}(X)_1$ of variance $s^2_1 \gg 1$ and a weak component $\phi_{(\ell)}(X)_2$ of variance $s^2_2 \ll 1$. In this representation, $\phi_{(\ell)}(x_{ood})$ is distant from $D$ to the training domain in the weak direction $\phi_{(\ell)}(X)_2$, resulting in $\text{Var}[h_{\theta}(x_{ood})] \geq \gamma D^2 \sigma^2_2 \simeq \frac{\gamma D^2}{s_2^2} \gg 1$.}
\label{illustration}
\end{figure}

\subsection{Verification on a synthetic dataset}

We consider the "two-moons" binary classification dataset and use a three hidden-layer network with 100 neurons each and ReLU activations as base network. We use a linear activation for the output layer and the mean squared error as loss function. As we consider a classification task, uncertainty can be obtained with a single network through the formula: $\min\left( h_{\theta}(x)^2, (1-h_{\theta}(x))^2 \right)$. The DARE uncertainty is the sum of the individual network members uncertainty plus the prediction variance between members.

Figure \ref{intern-analysis} presents the predicted uncertainty heat-map for one DARE network as well as the internal layer representations. We observe that the OOD component values grow from one layer to the next. A correspondence between features with low variance for the training data and large weights can be clearly observed. In the last hidden layer (layer 3), the OOD components are large in direction of low training variance (components 80 to 100) to which correspond large weights. This observation explains the large uncertainty score for the OOD example.


\begin{figure*}
    \centering
    \begin{minipage}{0.35\textwidth}
        \centering
        \includegraphics[width=\linewidth]{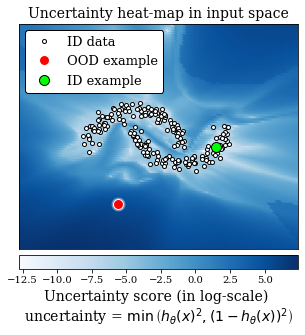} 
    \end{minipage}
    \begin{minipage}{0.64\textwidth}
        \centering
        \includegraphics[width=\linewidth]{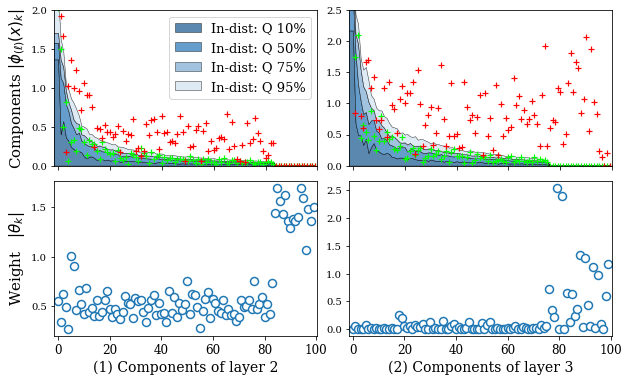}
    \end{minipage} \hfill
    \caption{\textbf{Internal analysis of a large weights network} The uncertainty produced by one member of an anti-regularized ensemble is presented on the left. On the right, the two figures on top present the expression of the training distribution in the three hidden layers (in blue) compared to the representation of one OOD example (in red). The components are sorted in descending order of training variance. The two bottom figures present the average weight in front of each component, i.e. the mean weights that multiply the layer components to produce the next layer representation.}
    \label{intern-analysis}
\end{figure*}

\section{Related Works}
\label{related}


Increasing ensemble diversity has been an enduring paradigm since the early days of the ensemble learning research. At first, diversity was seen as a key feature for improving the generalization ability of the ensembles \cite{brown2005managingdiversityensemble,kuncheva2003DiversityEnsemble,liu1999negativecorrelation,zhang2008ConstraintProjectionEnsemble}. Then, with the growing interest in uncertainty quantification, the primary objective of ensemble diversity becomes to produce good uncertainty estimates. In this perspective, a first category of methods propose to increase diversity by using diverse architectures or training conditions among an ensemble of deep neural networks \cite{Lakshminarayanan2017DeepEnsemble,Wen2020batchensemble,Wenzel2020hyperDeepEnsemble,Zaidi2021NIPSNES}. The underlying idea is that the diversity of architecture or local minima reached by the different networks induces a diversity of predictions. Another category of methods proposes to explicitly impose a diversity constraint in the loss function of the networks. The loss function is then written $\mathcal{L} + \lambda \mathcal{P}$ where $\mathcal{L}$ is the loss for the task (e.g. mean squared error or negative log-likelihood (NLL)), $\mathcal{P}$ is a penalty term which decreases with the diversity of the ensemble and $\lambda$ is the trade-off parameter between the two terms. Three kinds of penalization are distinguished in the literature. The first kind makes use of training data to compute the penalty term. It includes the Negative Correlation method (NegCorr) \cite{shui2018negativecorrelation,zhang2020NegCorr} which applies the penalization from \cite{liu1999negativecorrelation} to deep ensembles to enforce a negative correlation between the errors of the networks on the training set. Similarly, \cite{ross2020EnsemblesLocallyIndependant} imposes an orthogonality constraint between the gradients of the ensemble members on training data. Penalizing the similarity between hidden representations of the networks has also been proposed by \cite{rame2021diceUncertainty,Sinha2021DIBS} using adversarial training. The second kind of penalization refers to contrastive methods that enforces diversity on potential OOD instances rather than training data. This avoids the issue of being over-conditioned by the training domain that can be encountered by previous methods. In this category, several methods suppose that an unlabeled sample containing OOD is available, \cite{pagliardini2022DBAT,Tifrea2022semiSupervisedOOD}. Others avoid this restrictive assumption and simulate potential OOD with random uniform data \cite{Jain2020MOD,Mehrtens2022MODplus} or instances localized around the training data \cite{Segonne2022OODpseudoInputs}. In the last approach, considered by Anchored-Networks (AnchorNet) \cite{pearce2018AnchorNetwork} and Repulsive Deep Ensemble (RDE) \cite{Angelo2021RepulsiveDeepEnsemble}, the penalization $\mathcal{P}$ is a function of the network's parameters which forces the networks to reach local minima spaced from each other in parameters space. Our proposed DARE approach relates particularly to these two methods. Our assumption is that imposing weights diversity has more chance to obtain a global output diversity rather than imposing it on specific data regions as done by the two previous kind of penalization. Anchored-Networks appears to be an efficient tool, for instance, in the detection of corrupted data \cite{ulmer2020trustIssue}, however, it is very hard to set the right anchors and trade-off parameter \cite{scalia2020moleculeOODuncertainty}. Large initial variance can lead to large weight variance but may not converge to accurate model in-distribution. The DARE approach is more practical as it starts to increase the weights after reaching an acceptable loss threshold which ensures to fit the training data.

\section{Experiments}
\label{expe}

The experiments are conducted on both regression and classification datasets. In the majority of previous works, OOD uncertainty quantification is studied in the perspective of OOD detection in the classification setting where examples from other classes / datasets are considered as OOD \cite{Angelo2021RepulsiveDeepEnsemble,Lakshminarayanan2017DeepEnsemble,Liu2022DistanceAwarness,vanAmersfoort2020DUQ}. For regression, few attempts of uncertainty quantification on shifted datasets have been conducted: \cite{Jain2020MOD} separates male and female faces for age prediction dataset and \cite{foong2019inbetweenUncertainty,Jain2020MOD,Segonne2022OODpseudoInputs} propose OOD version of several UCI regression datasets \cite{Dua2019UCI}. In this work, we propose a novel regression setup for uncertainty estimations on shifted domains based on the CityCam dataset \cite{Zhang2017WebCamT} which has been used in several domain adaptation regression experiments \cite{deMathelin2020WANN,Zhao2018MultiSourceDANN}. Our setup models real-life domain shift scenarios where uncertainty quantification is challenged and offers a clear visual understanding of the domain shifts (cf Figure \ref{citycam-setup}). For the classification experiments, we consider the OOD detection setup developed in \cite{Angelo2021RepulsiveDeepEnsemble}.  The source code of the experiments is available on GitHub\footnote{\url{https://github.com/antoinedemathelin/DARE}}.



\subsection{Regression experiments on CityCam}

We consider the CityCam dataset \cite{Zhang2017WebCamT}. This dataset is composed of images recorded from different traffic cameras. The task consists in counting the number of vehicles present on the image, which is useful, for instance, to control the traffic in the city. To get relevant features for the task, we use the features of the last layer of a ResNet50 \cite{He2016ResNet} pretrained on ImageNet \cite{deng2009imagenet}. We propose three different kinds of domain shift:

\begin{figure}[ht]
\center
\includegraphics[width=0.6\linewidth]{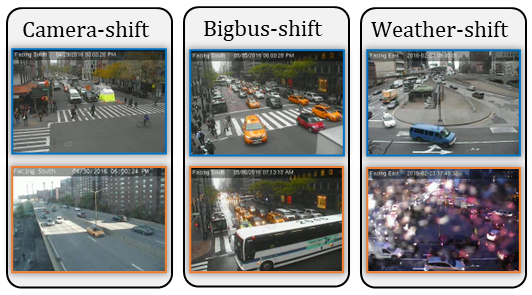}
\caption{\textbf{CityCam experimental setups}. The top blue images correspond to in-distribution examples and bottom orange images to OOD examples.}
\label{citycam-setup}
\end{figure}

\textbf{1. Camera-shift}: This experiment uses the images from 10 cameras in the CityCam dataset. For each round, 5 cameras are randomly selected as in-distribution while the 5 remaining cameras are considered as out-of-distribution. 



\textbf{2. Bigbus-shift}: Images marked as "big-bus" referring to the fact that a bus appears and masks a significant part of the image \cite{Zhang2017WebCamT} are used to create the OOD dataset.


\textbf{3. Weather-shift}: Blurry images caused by water drops landed on the camera are used as OOD dataset.


These three experiments model real-life uncertainty quantification problems as the generalization of uncertainty estimates to unseen domains (camera-shift), the robustness to changes in data acquisition (weather-shift) and the detection of rare abnormal event (bigbus-shift). Further details on these experimental setups are provided in Appendix.

We consider the following competitors: \textbf{Deep-Ensemble} (\textbf{DE}) \cite{Lakshminarayanan2017DeepEnsemble}, \textbf{NegCorr} \cite{shui2018negativecorrelation}, \textbf{AnchorNet} \cite{pearce2018AnchorNetwork}, \textbf{MOD} \cite{Jain2020MOD} and \textbf{RDE} \cite{Angelo2021RepulsiveDeepEnsemble}. All are deep ensemble methods which focus on increasing the diversity among members. AnchorNet, NegCorr and MOD introduce a penalty term in the loss function multiplied by a trade-off parameter $\lambda$. The trade-off $\lambda$ and the anchor initialization parameter $\sigma$ for AnchorNet are selected through hold-out validation, as suggested in \cite{Jain2020MOD,pearce2018AnchorNetwork}. The kernel bandwidth for RDE is adaptive and chosen with the median heuristic as suggested in the corresponding work \cite{Angelo2021RepulsiveDeepEnsemble}. The parameter $\tau$ of DARE is chosen equal to $1 + \delta$ times the DE validation loss. The underlying heuristic is to maintain the in-distribution performance level close to that of the vanilla DE. However, a small penalty $\delta > 0$ should be accepted to allow the weight increase. In these experiments, we intuitively select $\delta = 0.25$ without further fine-tuning. An ablation study of the impact of the $\delta$ parameter on the DARE performances is presented in Appendix.


The experiments are performed with ensemble of $5$ fully-connected networks with $3$ hidden layers of $100$ neurons each and ReLU activations. The Adam optimization algorithm is used with learning rate $0.001$ \cite{Kingma2014Adam}. The batch size is chosen equal to $128$. The experiments are repeated $5$ times to compute standard deviations for the scores. We use the NLL defines in \cite{Lakshminarayanan2017DeepEnsemble} as loss function.

Choosing the right number of optimization epochs to perform is not trivial. Previous works have used a fixed number of epochs \cite{pearce2018AnchorNetwork} or early stopping \cite{Jain2020MOD}. For DARE, a natural checkpoint strategy is to save the network if the validation loss of the current epoch is below the threshold $\tau$. This process ensures a targeted level of in-distribution performance and allows the weight increase during several epochs. For the competitors, we 
observed that restoring the weights of the best validation loss epoch provides the best results.


\begin{table}[ht]
\scriptsize
\begin{tabularx}{\linewidth}{X|YY|YY|YY}
\toprule
& \multicolumn{2}{c|}{Camera-shift } & \multicolumn{2}{c|}{Bigbus-shift} & \multicolumn{2}{c}{Weather-shift } \\
Methods  &                   NLL &                   ECE &                   NLL &                   ECE &                   NLL &                   ECE \\
\midrule
DE        &  6.04 (0.40) &  0.28 (0.02) &  3.66 (0.17) &  0.28 (0.01) &   2.25 (0.09) &  0.12 (0.02) \\
AnchoreNet     &  6.03 (0.24) &  0.28 (0.02) &  3.86 (0.16) &  0.29 (0.01) &   2.25 (0.12) &  0.12 (0.03) \\
MOD                 &  6.24 (0.89) &  0.29 (0.02) &  3.66 (0.21) &  0.29 (0.01) &   2.32 (0.17) &  0.15 (0.03) \\
NegCorr &  5.97 (0.54) &  0.29 (0.03) &  4.02 (0.27) &  0.29 (0.01) &   2.30 (0.10) &  0.14 (0.02) \\
RDE                 &  5.05 (0.25) &  0.27 (0.02) &  3.67 (0.11) &  0.28 (0.00) &   \textbf{2.22 (0.14)} &  0.13 (0.04) \\
DARE                 &  \textbf{3.64 (0.34)} &  \textbf{0.23 (0.03)} &  \textbf{2.99 (0.13)} &  \textbf{0.25 (0.01)} &   2.33 (0.13) &  \textbf{0.07 (0.04)} \\
\bottomrule
\end{tabularx}
\caption{\textbf{Out-of-distribution NLL and ECE for CityCam}}
\label{citycam-table1}
\end{table}


\begin{table}[ht]
\scriptsize
\center
\begin{tabularx}{\linewidth}{X|YYYYYY}
\toprule
{} &           DE &    AnchorNet &          MOD &      NegCorr &          RDE &         DARE \\
\midrule
Cameras &  1.03 (0.08) &  \textbf{1.00 (0.07)} &  1.04 (0.06) &  1.04 (0.06) &  1.06 (0.06) &  1.18 (0.09) \\
Bigbus  &  0.89 (0.03) &  \textbf{0.87 (0.02)} &  0.91 (0.04) &  0.89 (0.06) &  0.91 (0.03) &  1.06 (0.04) \\
Weather &  1.25 (0.04) &  \textbf{1.24 (0.05)} &  1.26 (0.05) &  1.26 (0.06) &  1.28 (0.05) &  1.27 (0.06) \\
\bottomrule
\end{tabularx}
\caption{\textbf{In-distribution NLL for CityCam}}
\label{citycam-table2}
\end{table}


\begin{figure}[ht]
\center
\includegraphics[width=0.8\linewidth]{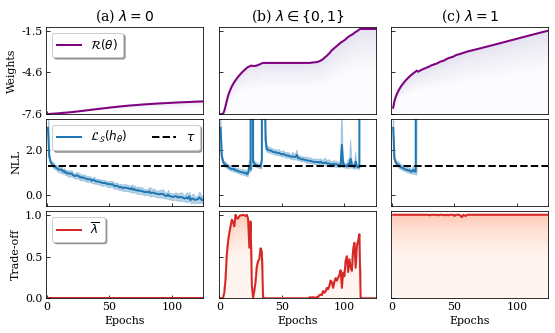}
\caption{\textbf{Training behaviors for different $\lambda$ settings}. Evolution of the weights, the training loss and the average trade-off parameter over the epochs ((a) $\equiv$ DE and (b) $\equiv$ DARE).}
\label{citycam-optim}
\end{figure}

To assess the ensemble quality for the regression experiments, we consider the NLL and expected calibration error (ECE) computed on the out-of-distribution dataset as used in the works \cite{Lakshminarayanan2017DeepEnsemble}, \cite{Mehrtens2022MODplus}. The results are averaged over $5$ runs and reported in Table \ref{citycam-table1}. We observe that DARE outperforms other uncertainty quantification methods in both term of NLL and ECE for the bigbus-shift and camera-shift experiments. For the weather-shift experiment, DARE provides the best ECE score. Moreover, these gains in uncertainty estimation come with little penalization of the performance in-distribution as reported in Table \ref{citycam-table2}: in the camera-shift experiment for instance, the in-distribution NLL of DARE is increased of only $0.18$ compare to the deep ensemble method, while the NLL out-of-distribution is decreased of $2.4$ (cf Table \ref{citycam-table1}). The impact of the $\lambda$ control in the DARE optimization is presented in Figure \ref{citycam-optim}, we observe that the setting $\lambda = 0$ (equivalent to DE) leads to a small weight increase whereas the setting $\lambda = 1$ leads to a loss explosion. The control setting $\lambda \in \{0, 1\}$ allows the weight increase while maintaining the loss above $\tau$.

\subsection{Classification Experiments}

We consider the experimental setup defines in \cite{Angelo2021RepulsiveDeepEnsemble} for OOD detection on Fashion-MNIST and CIFAR10. The MNIST dataset is used as OOD dataset for Fashion-MNIST and the SVHN dataset for CIFAR10. We extend the experiments by adding CIFAR10 as OOD for Fashion-MNIST and CIFAR100 as OOD for CIFAR10. Thus, for both experiments, OOD detection is performed on both "Near-OOD" and "Far-OOD" datasets \cite{Liu2022DistanceAwarness}.

We first observed that a direct application of DARE to this setup leads to negative results. The first issue comes from the softmax activation at the end layer, which cancels the effect of increasing the weights. Indeed, the softmax activation inverses the correlation between large outputs and high uncertainty, resulting in over-confidence for OOD data. A second issue comes from the use of convolutional layers in the CIFAR10 experiment. These layers have fewer parameters than fully-connected ones, which reduces the possibility of increasing the weights in front of weakly activated neurons. 

To overcome these difficulties, considering that DARE is efficient for fully-connected networks and linear end-activation, we remove the softmax activation at the end-layer of the DARE networks and use the mean squared error. Uncertainty scores are then obtained by computing the distance between the predicted logits and the one-hot-encoded vector of the predicted class. For the CIFAR10 experiment, we use a pretrained ResNet32 and replace the end layer by deep ensemble of fully-connected networks. This setup is more practical, as only one ResNet32 is required, which reduces the need of computational resources \cite{lee2015multihead}.

The obtained results are reported in Table \ref{class-table} for DARE and the competitors. To fully evaluate the impact of the DARE optimization, we add the results obtained with a Deep Ensemble trained with the mean squared error (DE (MSE)) which is equivalent to a DARE with $\lambda = 0$. We train $5$ networks in each ensemble and repeat the experiments 5 times. The AUROC metric is used, computed with the the uncertainty score defined previously for DARE and DE (MSE) and the entropy for the other methods. We observe that DARE globally improves the OOD detection. For instance, in Fashion-MNSIT, we observe an improvement of $8$ points on CIFAR10 and $34$ points on MNIST, with a lost of only $2$ points of in-distribution accuracy.

\begin{table}[ht]
\scriptsize
\begin{tabularx}{\linewidth}{X|YYY|YYY}
\toprule
& \multicolumn{3}{c|}{CIFAR10} & \multicolumn{3}{c}{Fashion-MNIST} \\
Methods & SVHN & CIFAR100 & Accuracy &   CIFAR10 &  MNIST &     Accuracy \\
\midrule
DE (NLL)        &          90.9 (0.4) &              86.4 (0.2) &  \textbf{91.8 (0.1)} &            89.7 (0.9) &    62.7 (6.2) &  \textbf{89.2 (0.2)} \\
AnchorNet     &          91.0 (0.3) &              \textbf{86.5 (0.2)} &  \textbf{91.8 (0.0)} &            88.8 (1.1) &    68.7 (6.2) &  89.1 (0.2) \\
MOD                 &          91.3 (0.3) &              86.3 (0.3) &  91.7 (0.1) &            89.4 (1.7) &    60.8 (2.7) &  88.7 (0.4) \\
NegCorr &          91.3 (0.4) &              86.3 (0.4) &  91.7 (0.1) &            91.5 (0.8) &    68.9 (4.5) &  86.1 (0.6) \\
RDE                 &          91.2 (0.5) &              86.4 (0.3) &  \textbf{91.8 (0.1)} &            90.1 (0.9) &    70.9 (5.8) &  89.1 (0.3) \\
DE (MSE)     &          85.9 (1.2) &              77.7 (0.8) &  91.7 (0.1) &            96.5 (0.5) &    93.0 (5.3) &  88.6 (0.1) \\
DARE                 &          \textbf{92.6 (0.7)} &              82.7 (0.5) &  \textbf{91.8 (0.1)} &            \textbf{97.7 (0.5)} &    \textbf{97.4 (1.3)} &  87.2 (0.2) \\
\bottomrule
\end{tabularx}
\caption{\textbf{OOD detection results}. AUROC scores and ID accuracy are reported}
\label{class-table}
\end{table}

\section{Limitations and Perspectives}

For now, the efficiency of DARE is limited to fully-connected neural network with piece-wise linear activation. We have seen, however, that DARE can benefit to a final fully-connected network head placed on top of deep features obtained with convolutions. Thanks to the practical aspect of DARE, the method can be combined to other deep ensemble or OOD detection methods. One can use a specific training process and then apply DARE afterwards to increase diversity. Future work will consider a Bayesian version of DARE by adding gaussian noise with increasing variance to the weights of pretrained networks.

\clearpage

\appendix

{\parindent0pt

\section{Proof of Theorem 1}

We first remind the assumptions and notations used for the Theorem:

\noindent \textbf{Notations}:
\begin{itemize}
    \item $X \in \mathbb{R}^{n \times p}$ and $y \in \mathbb{R}^{n}$ are respectively the input and output data which compose the training set $\mathcal{S}$.
    \item $\mathcal{H} = \{x \to x^T \theta; \theta \in \mathbb{R}^p\}$ is the ensemble of linear networks.
    \item $\mathcal{L}_{\mathcal{S}}(h_{\theta}) = \frac{1}{n} ||X \theta - y||^2_2$ is the mean squared error.
    \item $s^2 = (s_1^2, ..., s_p^2) \in {\mathbb{R}_+^*}^p$ is the diagonal of the matrix $\frac{1}{n} X^T X$.
    \item $\Theta_{\sigma^2}$ is the weights distribution of the anti-regularized ensemble $\mathcal{H}_{\tau} \subset \mathcal{H}$, with $\sigma^2$ the weights variance.
\end{itemize}

\noindent \textbf{Assumptions}:
\begin{equation}
\label{app-assum-bwe-0}
    \exists \, \theta^* \in \mathbb{R}^p; X \theta^* = y
\end{equation}
For any $h_{\theta} \in \mathcal{H}_{\tau}$:
\begin{equation}
\label{app-assum-bwe-1}
    \theta \sim \Theta_{\sigma^2};  \;  \;\mathbb{E}[\theta] = \theta^*, \; \text{Cov}(\theta) = \text{diag}(\sigma^2)
\end{equation}
\begin{equation}
\label{app-assum-bwe-2}
\mathbb{E}\left[\mathcal{L}_{\mathcal{S}}(h_\theta) \right] \leq  \delta \, \tau
\end{equation}
Where $\delta > 0$ and $\textnormal{diag}(\sigma^2)$ is the diagonal matrix of values $\sigma^2 \in \mathbb{R}_+^p$ verifying:
\begin{equation}
\label{app-assum-bwe-3}
    \sigma^2 = \argmax_{\sigma^2 = (\sigma_1^2, ..., \sigma_p^2)} \sum_{k=1}^p \log \left({\theta_k^*}^2 + \sigma_k^2 \right)
\end{equation}

\begin{theorem}
\label{app-thm1}
There exist a constant $C > 0$ such that for any $k \in [|1, p|]$, the variance of the $k^{th}$ weight component is expressed as follows:
\begin{equation}
\label{app-thm-wf-eq}
    \sigma_k^2 = \max \left[ \frac{C \, \delta \, \tau}{s_k^2} - {\theta_k^*}^2, 0 \right]
\end{equation}
\end{theorem}

\begin{proof}
Let's introduce the variable $z \sim \mathcal{N}(0, \text{diag}(\sigma^2))$, verifying: $\theta = \theta^* + z$. From Assumption \ref{app-assum-bwe-2}, we derive that:
\begin{equation}
\begin{split}
    \mathcal{L}_{\mathcal{S}}(h_{\theta}) & = \frac{1}{n} || X \theta - y ||_2^2 \\
    & = \frac{1}{n} || X (\theta^* + z) - y ||^2_2 \\
    & = \frac{1}{n} || X \theta^* - y + X z ||^2_2 \\
    & = \frac{1}{n} || X z ||^2_2 \; \; \text{(by definition of $\theta^*$)} \\
    & = \frac{1}{n} \sum_{i=1}^n \sum_{k=1}^p \sum_{j=1}^p X_{ik} X_{ij} z_k z_j \\
\end{split}
\end{equation}

Thus, we have:
\begin{equation}
\begin{split}
    {\mathbb{E}}[\mathcal{L}_{\mathcal{S}}(h_{\theta})] & = \frac{1}{n} \sum_{i=1}^n \sum_{k=1}^p \sum_{j=1}^p X_{ik} X_{ij} {\mathbb{E}}[z_k z_j] \\
    & = \frac{1}{n} \sum_{i=1}^n \sum_{k=1}^p X_{ik}^2  \sigma_k^2 \; \; \text{(by definition of $z$)} \\
    & = \sum_{k=1}^p \left( \frac{1}{n} \sum_{i=1}^n  X_{ik}^2 \right) \sigma_k^2 \\
    & = \sum_{k=1}^p s_k^2 \sigma_k^2
\end{split}
\end{equation}

Combining this results with Assumption \ref{app-assum-bwe-3}, we show that $\sigma^2$ verifies:
\begin{equation}
\label{app-proof-4}
\begin{split}
    \underset{\sigma^2 \in \mathbb{R}_+^p}{\textnormal{maximize}} &  \; \; \; \sum_{k=1}^p \log(\sigma_k^2 + {\theta^*_k}^2) \\
    \textnormal{subject to} & \; \; \;  \sum_{k=1}^p s_k^2 \sigma_k^2 \leq \delta \, \tau
\end{split}
\end{equation}

This expression is a "water-filling" problem (cf \cite{boyd2006convex} Example 5.2) with weighted constraint. The inequality constraint $\sum_{k=1}^p s_k^2 \sigma_k^2 \leq \delta \, \tau$ can be written as an equality constraint, as any increase of $\sigma_k^2$ induces an increase of the objective function.

To solve this problem, we introduce the Lagrange multipliers $\mu \in \mathbb{R}^p_+$ for the constraints $\sigma^2 \geq 0$ and the multiplier $\alpha \in \mathbb{R}$ for the constraint $\sum_{k=1}^p s_k^2 \sigma_k^2 = \delta \, \tau$. By considering the Lagrangian as a function of $\sigma^2$, the KKT conditions are then written:
\begin{gather}
    \frac{-1}{{\theta_k^*}^2 + \sigma_k^2} + \alpha s_k^2  - \mu_k = 0 \; \forall \, k \in [|1, p|] \\
    \sigma_k^2 \geq 0, \; \mu_k \geq 0, \; \mu_k \sigma_k^2 = 0 \; \forall k \in [|1, p|] \; \; \text{and} \; \; \sum_{k=1}^p s_k^2 \sigma_k^2 = \delta \, \tau
\end{gather}

Leading to:
\begin{gather}
    \sigma_k^2 \left( \frac{-1}{{\theta_k^*}^2 + \sigma_k^2} + \alpha s_k^2 \right) = 0 \; \forall \, k \in [|1, p|] \\
    \sigma_k^2 \geq 0, \; \alpha s_k^2 \geq \frac{1}{{\theta_k^*}^2 + \sigma_k^2}, \; \forall k \in [|1, p|] \; \; \text{and} \; \; \sum_{k=1}^p s_k^2 \sigma_k^2 = \delta \, \tau
\end{gather}

Then, for any $k \in [|1, p|]$, we have:
\begin{gather}
    \sigma_k^2 = 0 \; \text{or} \; \sigma^2 = \frac{1}{\alpha s_k^2} - {\theta_k^*}^2
\end{gather}
As $\sigma_k^2 \geq 0$, we deduce that:
\begin{gather}
    \sigma_k^2 = \max \left[\frac{1}{\alpha s_k^2} - {\theta_k^*}^2, 0 \right]
\end{gather}
With $\alpha$ verifying:
\begin{equation}
\begin{split}
     \sum_{k=1}^p s_k^2 \max \left[\frac{1}{\alpha s_k^2} - {\theta_k^*}^2, 0 \right] = \delta \, \tau \\
     \sum_{k=1}^p \max \left[\frac{1}{\alpha \delta \, \tau} - \frac{s_k^2 {\theta_k^*}^2}{\delta \, \tau}, 0 \right] = 1
\end{split}
\end{equation}
By defining $C = \frac{1}{\alpha \delta \, \tau}$, we have:
\begin{equation}
\begin{split}
     \sum_{k=1}^p \max \left[C - \frac{s_k^2 {\theta_k^*}^2}{\delta \, \tau}, 0 \right] = 1
\end{split}
\end{equation}
Which imposes $C>0$.

We conclude then, that there exists $C > 0$ such that, for any $k \in [|1, p|]$:
\begin{gather}
    \sigma_k^2 = \max \left[\frac{C \delta \, \tau}{s_k^2} - {\theta_k^*}^2, 0 \right]
\end{gather}
    
\end{proof}

\section{Proof of Corollary 1}

\begin{corollary}
\label{app-cor}
Let $\mathcal{H}_{\tau}$ be the large weights ensemble defined by Assumptions \ref{app-assum-bwe-1}, \ref{app-assum-bwe-2}, \ref{app-assum-bwe-3} and $x \in \mathbb{R}^p$ an input data, the variance of prediction for $h_{\theta} \in \mathcal{H}_{\tau}$ is:
\begin{equation}
    \underset{\theta \sim \Theta_{\sigma^2}}{\textnormal{Var}}\left[h_{\theta}(x) \right] = \sum_{k=1}^p x_k^2 \max \left[ \frac{C \, \delta \, \tau}{s_k^2} - {\theta_k^*}^2, 0 \right]
\end{equation}
\end{corollary}

\begin{proof}
Let's consider $x \in \mathbb{R}^p$ and $h_{\theta} \in \mathcal{H}_{\tau}$:
\begin{equation}
\begin{split}
     \underset{\theta \sim \Theta_{\sigma^2}}{\textnormal{Var}}\left[h_{\theta}(x) \right] & = \underset{\theta \sim \Theta_{\sigma^2}}{\textnormal{Var}}\left[ \sum_{k=1}^p x_k \theta_k \right] \\
     & = \sum_{k=1}^p x^2_k \sigma_k^2 \; \; \text{(by Assumption \ref{app-assum-bwe-1})}\\
\end{split}
\end{equation}

\end{proof}

\section{Proof of Theorem 2}

\textbf{Notations}:
\begin{itemize}
    \item $X \in \mathbb{R}^{n \times p}$ and $y \in \mathbb{R}^{n}$ are respectively the input and output data which compose the training set $\mathcal{S}$.
    \item $\mathcal{H}_{\tau}$ is an anti-regularized ensemble of fully-connected deep neural networks with $L$ hidden layers and network weights distributed according to $\Theta_{\sigma^2}$ in $\mathbb{R}^d$.
    \item For any $\ell \in [|0, L+1|]$, $p_{\ell}$ is the number of neurons of the $\ell^{th}$ layer.
    \item $g: \mathbb{R} \to \mathbb{R}$ is the activation function applied at each hidden layer.
    \item For every $\ell \in [|0, L|]$, we denote $\phi_{(\ell)}(X) \in \mathbb{R}^{n \times p}$ the representation of the $\ell^{th}$ layer, with $\ell = 0$ the input layer.
    \item For any $j \in [|1, p_{\ell+1}|]$, we denote $\theta_{(\ell, j)} \in \mathbb{R}^p$ the network weights between the layer $\ell$ and the $j^{th}$ component of the layer $\ell + 1$.
    \item For any $\ell \in [|0, L|]$, $\phi_{(\ell)}^*(X) = \mathbb{E}[\phi_{(\ell)}(X)]$ and $s^2_{(\ell)}$ is the diagonal of the matrix $\frac{1}{n} \phi_{(\ell)}^*(X)^T \phi_{(\ell)}^*(X)$.
\end{itemize}

\textbf{Assumptions}:
\begin{equation}
\label{assum-ml-2}
\begin{gathered}
\textnormal{For any random variable $Z$ in $\mathbb{R}$ there exists $\alpha > 0$ such that:} \\ \textnormal{ $\text{Var}[g(Z)] > \alpha \text{Var}[Z]$.}
\end{gathered}
\end{equation}
\begin{equation}
\label{assum-ml-3}
\begin{gathered}
\textnormal{There is $\delta > 0$ such that, for any $\ell \in [|0, L|]$ and any $j \in [|1, p_{\ell+1}|]$,} \\ \textnormal{$\phi_{(\ell)}(X)$ and $\theta_{(\ell, j)}$ verify Assumptions \ref{app-assum-bwe-1}, \ref{app-assum-bwe-2}, \ref{app-assum-bwe-3} with} \\ \textnormal{$\theta \equiv \theta_{(\ell, j)}$, $X \equiv \phi_{(\ell)}^*(X)$, $y \equiv \phi_{(\ell)}^*(X) \theta_{(\ell, j)}^*$ and $\theta^* \equiv \theta_{(\ell, j)}^* \equiv \mathbb{E}[\theta_{(\ell, j)}]$}
\end{gathered}
\end{equation}

\begin{theorem}
\label{app-thm-multi}
There exists a constant $\gamma>0$ such that, for any $h_{\theta} \in \mathcal{H}_{\tau}$ and $x \in \mathbb{R}^p$, the prediction variance verifies:
\begin{equation*}
    \underset{\theta \sim \Theta_{\sigma^2}}{\textnormal{Var}}\left[h_{\theta}(x) \right] \geq \gamma \sum_{\ell = 0}^{L} \langle \phi_{(\ell)}^*(x)^2, \bar{\sigma}_{(\ell)}^2 \rangle
\end{equation*}
With, $ \bar{\sigma}_{(\ell)}^2 = \sum_{j=1}^p \max \left[ \frac{C_{(\ell, j)}}{s^2_{(\ell)}} - {\theta_{(\ell, j)}^*}^2, 0 \right]$; $\, C_{(\ell, j)} \in \mathbb{R}^*_+$
\end{theorem}

\begin{proof}
Let's consider $x \in \mathbb{R}^p$ and $h_{\theta} \in \mathcal{H}_{\tau}$:

Considering Assumption \ref{assum-ml-3}, we apply Theorem \ref{app-thm1} to $\phi_{(\ell)}(x)_j$ for any $\ell \in [|0, L|]$ and any $j \in [|1, p_{\ell+1}|]$, such that:
\begin{equation}
\begin{split}
     \text{Var}[\theta_{(\ell, j)}] = \sigma_{(\ell, j)}^2 = \max \left[ \frac{C_{(\ell, j)}}{s^2_{(\ell)}} - {\theta_{(\ell, j)}^*}^2, 0 \right]
\end{split}
\end{equation}
With $C_{(\ell, j)} \in \mathbb{R}^*_+$. Notice that the $\max$ function is applied element-wise.

For any $\ell \in [|0, L-1|]$ and any $j \in [|1, p_{\ell + 1}|]$, the following inequalities are verified:
\begin{equation}
\label{rec}
\begin{split}
     {\textnormal{Var}}\left[\phi_{(\ell+1)}(x)_j \right] & = {\textnormal{Var}}\left[g\left(\phi_{(\ell)}(x) \theta_{(\ell, j)}\right) \right] \\
     & \geq \alpha {\textnormal{Var}}\left[\phi_{(\ell)}(x) \theta_{(\ell, j)} \right] \quad  \textnormal{(by Assumption \ref{assum-ml-2})} \\
     & \geq \alpha {\textnormal{Var}}\left[ \sum_{k=1}^{p_{\ell}}  \phi_{(\ell)}(x)_k \theta_{(\ell, j, k)} \right] \\
     & \geq \alpha  \sum_{k=1}^{p_{\ell}}  {\textnormal{Var}}\left[\phi_{(\ell)}(x)_k \theta_{(\ell, j, k)} \right] \; \text{($\phi_{(\ell)}(x)_k \theta_{(\ell, k, j)}$ independents)} \\
     & \geq \alpha  \sum_{k=1}^{p_{\ell}}  {\textnormal{Var}}\left[\phi_{(\ell)}(x)_k\right] {\textnormal{Var}}\left[\theta_{(\ell, j, k)} \right]
     + \mathbb{E}\left[\phi_{(\ell)}(x)_k\right]^2 {\textnormal{Var}}\left[\theta_{(\ell, j, k)} \right] \\
     & \quad \quad \quad \;  + \mathbb{E}\left[\theta_{(\ell, j, k)} \right]^2 {\textnormal{Var}}\left[\phi_{(\ell)}(x)_k\right]  \\
     & \geq \alpha \sum_{k=1}^{p_{\ell}} {\textnormal{Var}}\left[\phi_{(\ell)}(x)_k\right] \left( \sigma^2_{(\ell, j, k)} + {\theta_{(\ell, j, k)}^*}^2 \right) + \phi^*_{(\ell)}(x)^2_k \sigma^2_{(\ell, j, k)}  \\
     & \geq \alpha \langle \phi^*_{(\ell)}(x)^2, \sigma_{(\ell, j)}^2  \rangle + \alpha \sum_{k=1}^{p_{\ell}} {\textnormal{Var}}\left[\phi_{(\ell)}(x)_k\right] \left( \sigma^2_{(\ell, j, k)} + {\theta_{(\ell, j, k)}^*}^2 \right) \\
     & \geq \alpha \langle \phi^*_{(\ell)}(x)^2, \sigma_{(\ell, j)}^2  \rangle + \alpha \gamma_{(\ell, j)} \sum_{k=1}^{p_{\ell}} {\textnormal{Var}}\left[\phi_{(\ell)}(x)_k\right]
\end{split}
\end{equation}
With $\gamma_{(\ell, j)} \geq \underset{k \in [|1, {p_{\ell}}|]}{\min} \left( \sigma^2_{(\ell, j, k)} + {\theta_{(\ell, j, k)}^*}^2 \right) = \underset{k \in [|1, {p_{\ell}}|]}{\min} \max \left[ \frac{C_{(\ell, j)}}{s^2_{(\ell)}}, {\theta_{(\ell, j)}^*}^2 \right] > 0$

The fourth inequality comes from the variance decomposition of the product of two independent variables: $X \indep Y \implies \text{Var}[X Y] = \text{Var}[X] \text{Var}[Y] + \text{Var}[X] \mathbb{E}[Y]^2 + \text{Var}[Y] \mathbb{E}[X]^2$.

The last inequality in (\ref{rec}) provides an expression of recurrence for ${\textnormal{Var}}\left[\phi_{(\ell)}(x)_j \right]$. We notice that, ${\textnormal{Var}}\left[\phi_{(0)}(x)_j \right] = {\textnormal{Var}}\left[x_j \right] = 0$ for any $j \in [|1, p_0|]$. For the last layer, we have:
\begin{equation}
{\textnormal{Var}}\left[h_{\theta}(x) \right] = {\textnormal{Var}}\left[\phi_{(L)}(x) \theta_{(L, 1)} \right] \geq \langle \phi^*_{(L)}(x)^2, \sigma_{(L, 1)}^2  \rangle + \gamma_{(L, 1)} \sum_{k=1}^{p_{L}} {\textnormal{Var}}\left[\phi_{(L)}(x)_k\right]
\end{equation}

We pose $\tilde{\gamma} = \underset{\ell \in [|0, L|]}{\min} \underset{j \in [|1, p_{\ell+1}|]}{\min} \gamma_{(\ell, j)}$ and $\gamma = \min[1, \tilde{\gamma}, \tilde{\gamma} \alpha^{L}] > 0$. We then deduce, by considering the recurence expression, that:
\begin{equation}
    {\textnormal{Var}}\left[h_{\theta}(x) \right] \geq \gamma \sum_{\ell = 0}^{L} \langle \phi_{(\ell)}^*(x)^2, \bar{\sigma}_{(\ell)}^2 \rangle
\end{equation}
With, $ \bar{\sigma}_{(\ell)}^2 =  \sum_{j=1}^{p_{\ell + 1}} \sigma^2_{(\ell, j)} = \sum_{j=1}^{p_{\ell + 1}} \max \left[ \frac{C_{(\ell, j)}}{s^2_{(\ell)}} - {\theta_{(\ell, j)}^*}^2, 0 \right]$

\end{proof}

\section{Synthetic Experiments}

\noindent We consider the two synthetic experiments : Two moons classification and 1D regression :

\begin{itemize}
    \item \textbf{Two Moons Classification} : $200$ random data points are generated from the two moons generator\footnote{\url{https://scikit-learn.org/stable/modules/generated/sklearn.datasets.make\_moons.html}} to form the training set. We consider the three uncertainty quantification methods: Deep Ensemble (NLL), Deep Ensemble (MSE) and Deep Anti-Regularized Ensemble (DARE). The first method implements a softmax activation function at the end layer and use the multi-classication NLL as loss function. The two others use a linear activation at the end layer and the mean squared error as loss function. A fully-connected network of three layers with $100$ neurons each and ReLU activations is used as base network. Each ensemble is composed of $20$ networks. The Adam \cite{Kingma2014Adam} optimization is used with learning rate $0.001$, batch size $32$ and $500$ epochs. For DARE, the threshold is set to $\tau = 0.001$. To compute the uncertainty scores, the entropy metric \cite{Lakshminarayanan2017DeepEnsemble} is used for Deep Ensemble (NLL) while Deep Ensemble (MSE) and DARE use the uncertainty metric for classification with MSE loss function defined in Appendix \ref{classif}. We report the uncertainty map produced by each method in Figure \ref{toy_classif}. We also report the results for the OOD detection task in Figure \ref{toy_classif_ood}. To produce this last results, we sample $50$ random points from the two moons generator which acts as validation data. We compute the uncertainty scores on the validation set and use the $95\%$ percentiles as threshold for OOD detection. Any data points below this threshold are considered as in-distribution data (in white) and the others as OOD (in dark blue). 
    \item \textbf{1D Regression} : We reproduce the synthetic univariate Regression experiment from \cite{Jain2016ALDiversity}. We consider the three methods: Deep Ensemble (NLL), Deep Ensemble (MSE) DARE. We use the gaussian NLL loss for regression defined in \cite{Lakshminarayanan2017DeepEnsemble} for Deep Ensemble (NLL) and DARE. The mean squared error is used as loss function for Deep Ensemble (MSE). The threshold for DARE is set to $\tau = 0.1$. For any data $x \in \mathbb{R}$, each method, return two values $\mu_x$ and $\sigma_x$ (cf \cite{Lakshminarayanan2017DeepEnsemble}). We report in Figure \ref{toy_reg} the confidence intervals $[\mu_x - 2 \sigma_x, \mu_x + 2 \sigma_x ]$ (in light blue) for any $x$ in the input range.
\end{itemize}

\begin{figure}[H]
    \centering
    \begin{minipage}{0.33\linewidth}
        \centering
        \includegraphics[width=\linewidth]{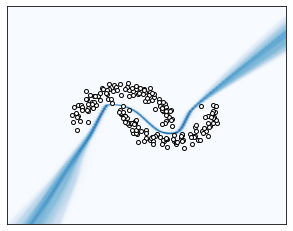} 
        \caption*{(a) Deep Ensemble (NLL)}
    \end{minipage}\hfill
    \begin{minipage}{0.33\linewidth}
        \centering
        \includegraphics[width=\linewidth]{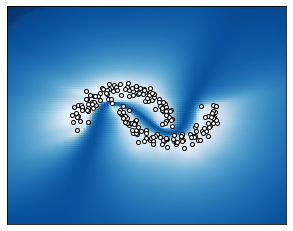} 
        \caption*{(b) Deep Ensemble (MSE)}
    \end{minipage}\hfill
    \begin{minipage}{0.33\linewidth}
        \centering
        \includegraphics[width=\linewidth]{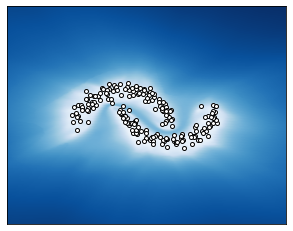} 
        \caption*{(c) DARE}
    \end{minipage}
    \caption{\textbf{Two-moons uncertainty estimation} Darker areas correspond to higher predicted uncertainty.}
    \label{toy_classif}
\end{figure}

\begin{figure}[H]
    \centering
    \begin{minipage}{0.33\linewidth}
        \centering
        \includegraphics[width=\linewidth]{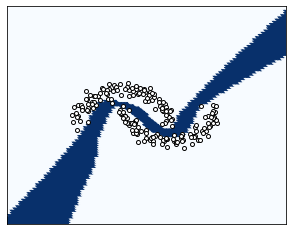} 
        \caption*{(a) Deep Ensemble (NLL)}
    \end{minipage}\hfill
    \begin{minipage}{0.33\linewidth}
        \centering
        \includegraphics[width=\linewidth]{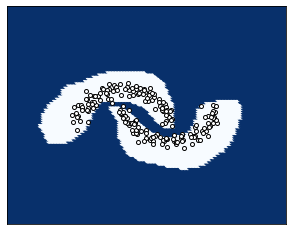} 
        \caption*{(b) Deep Ensemble (MSE)}
    \end{minipage}\hfill
    \begin{minipage}{0.33\linewidth}
        \centering
        \includegraphics[width=\linewidth]{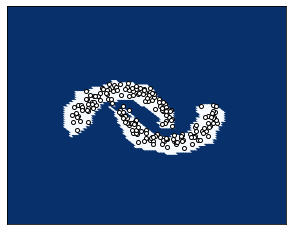} 
        \caption*{(c) DARE}
    \end{minipage}
    
    \caption{\textbf{OOD detection}. Data classified as OOD are in dark blue OOD while the ones classified as in-distribution are in white.}
    \label{toy_classif_ood}
\end{figure}

\begin{figure}[H]
    \centering
    \begin{minipage}{0.33\linewidth}
        \centering
        \includegraphics[width=\linewidth]{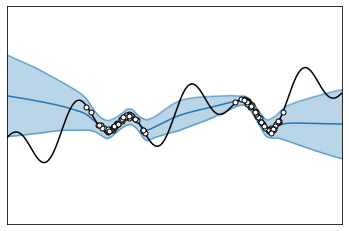} 
        \caption*{(a) Deep Ensemble (NLL)}
    \end{minipage}\hfill
    \begin{minipage}{0.33\linewidth}
        \centering
        \includegraphics[width=\linewidth]{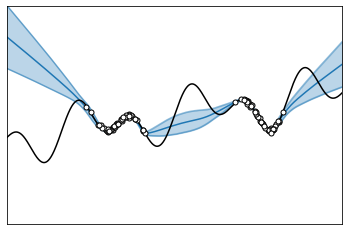} 
        \caption*{(b) Deep Ensemble (MSE)}
    \end{minipage}\hfill
    \begin{minipage}{0.33\linewidth}
        \centering
        \includegraphics[width=\linewidth]{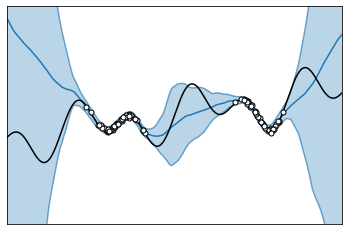} 
        \caption*{(c) DARE}
    \end{minipage}
    \caption{\textbf{1D Regression uncertainty estimation}. the confidence intervals for $\pm 2 \sigma$ are given in light blue.}
    \label{toy_reg}
\end{figure}

\section{CityCam Experiments}

We consider the three following experiments:

\begin{itemize}
    \item \textbf{Weather-shift}: in this experiment, we select the images of the three cameras n°164, 166 and 572 and use the images recorded during February the $23^{\text{th}}$. During this particular day, the weather significantly change between the beginning and the end of the day. Thus, to evaluate the robustness of the method to a change in the weather condition, we split the dataset in two subsets: we consider the images recorded before $2$ pm as in-distribution and the others as out-of-distribution. we observe a shift in the weather condition around $4$ pm which produces a visual shift in the images. After $4$ pm, the rain starts to fall and progressively damages the visual rendering of the images due to the reverberation. This is particularly the case for camera n°164, as the water drops have landed on the camera and blur the images. Each dataset is composed of around $2500$ images.
    \item \textbf{Camera-shift}: this experiment uses the images from 10 cameras in the CityCam dataset. For each round, 5 cameras are randomly selected as in-distribution while the 5 remaining cameras are considered as out-of-distribution. In average, each set is composed of around $20000$ images.
    \item \textbf{Bigbus-shift}: The CityCam dataset contains images marked as "big-bus" referring to the fact that a bus appears and masks an significant part of the image \cite{Zhang2017WebCamT}. We then select the 5 cameras for which some images are marked as "big-bus" and use this marker to split the dataset between in-distribution and out-of-distribution samples. The in-distribution set is composed of around $17000$ images while the other set contains around $1000$ images.
\end{itemize}

\begin{figure}[H]
    \centering
    \begin{minipage}{0.33\linewidth}
        \centering
        \includegraphics[width=\linewidth]{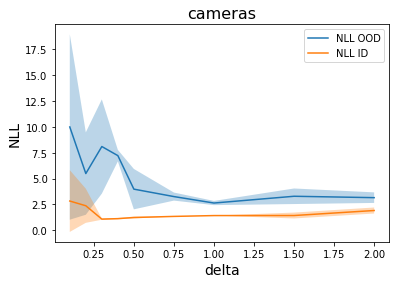} 
        \caption*{(a) Cameras-shift}
    \end{minipage}\hfill
    \begin{minipage}{0.33\linewidth}
        \centering
        \includegraphics[width=\linewidth]{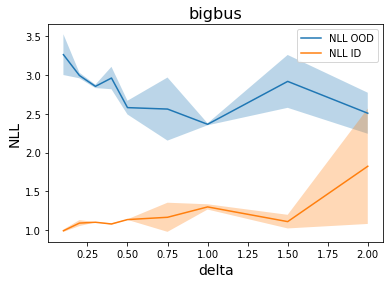} 
        \caption*{(b) Bigbus-shift}
    \end{minipage}\hfill
    \begin{minipage}{0.33\linewidth}
        \centering
        \includegraphics[width=\linewidth]{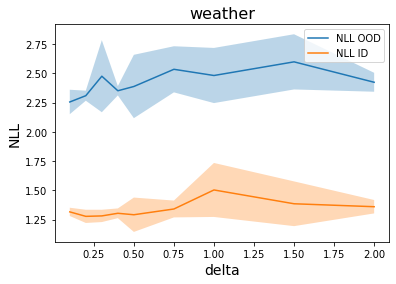} 
        \caption*{(c) Weather-shift}
    \end{minipage}
    \caption{\textbf{Ablation study for the $\delta$ parameter}. The DARE NLL is reported for different values of $\delta$ with $\tau = 1 + \delta$. We globally observe a decrease of the OOD NLL and a sligth increase of the in-distribution NLL when $\delta$ increases.}
    \label{nll-delta}
\end{figure}

\section{OOD Detection Experiments}
\label{classif}

For the OOD detection experiments, we follow the setup of \cite{Angelo2021RepulsiveDeepEnsemble}. For the Fashion-MNIST experiemnt, a three layers fully-connected network is use as base network, with ReLU activations. We use the Adam optimizer \cite{Kingma2014Adam} with learning rate $0.001$ and a batch size of $128$. The maximal number of epochs is set to $50$. For the CIFAR10 experiment, the same base network is used on top of a pretrained ResNet32. We consider the same optimization parameters.

For DARE and Deep Ensemble (MSE), the uncertainty score is computed as follows:

\begin{equation}
    \text{OOD-score}(x) = \frac{1}{M} \sum_{k=1}^M || h_k(x) - \widehat{y}_k ||_2^2 + \frac{1}{M} \sum_{k=1}^M \left|\left| h_k(x) - \frac{1}{M} \sum_{i=1}^M  h_i(x) \right|\right|_2^2 \end{equation}

With $M$ the number of network in the ensemble, $h_k$ the network members and $\widehat{y}_k$ the one-hot-encoded vector of the predicted class multiplied by the number of classes.

}

\end{document}